\documentclass{article}

\usepackage[preprint]{neurips_2025}

\usepackage[utf8]{inputenc} 
\usepackage[T1]{fontenc}    
\usepackage{hyperref}       
\usepackage{url}            
\usepackage{booktabs}       
\usepackage{amsfonts}       
\usepackage{nicefrac}       
\usepackage{microtype}      
\usepackage{xcolor}         
\usepackage{graphicx}
\usepackage{amsmath}
\usepackage{arydshln}

\usepackage{color-edits}
\usepackage{wrapfig}

\usepackage{amsmath}
\usepackage{amssymb}
\usepackage{mathtools}
\usepackage{amsthm}

\usepackage[capitalize,noabbrev]{cleveref}
\theoremstyle{plain}
\newtheorem{theorem}{Theorem}[section]

\theoremstyle{definition}

\theoremstyle{remark}

\addauthor{lu}{red}

\title{RePrompt: Reasoning-Augmented Reprompting for Text-to-Image Generation via Reinforcement Learning}

\author{%
  Mingrui Wu$^1$\thanks{Work done at microsoft. $\dag$Project Leader.},  Lu Wang$^2$$^{\dag}$, Pu Zhao$^2$$^{\dag}$,  Fangkai Yang$^2$$^{\dag}$, Jianjin Zhang$^2$, Jianfeng Liu$^2$, \\
  \textbf{Yuefeng Zhan$^2$, Weihao Han$^2$, Hao Sun$^2$, Jiayi Ji$^1$, Xiaoshuai Sun$^1$, Qingwei Lin$^2$,} \\ 
  \textbf{Weiwei Deng$^2$, Dongmei Zhang$^2$, Feng Sun$^2$, Qi Zhang$^2$, Rongrong Ji$^1$} \\
  $1$ Key Laboratory of Multimedia Trusted Perception and Efficient Computing, \\ Ministry of Education of China, Xiamen University, 361005, P.R. China\\
  $2$ Microsoft
}

\begin{document}

\maketitle

\begin{abstract}
  Despite recent progress in text-to-image (T2I) generation, existing models often struggle to faithfully capture user intentions from short and under-specified prompts. While prior work has attempted to enhance prompts using large language models (LLMs), these methods frequently generate stylistic or unrealistic content due to insufficient grounding in visual semantics and real-world composition. Inspired by recent advances in reasoning for language model, we propose RePrompt, a novel reprompting framework that introduces explicit reasoning into the prompt enhancement process via reinforcement learning. Instead of relying on handcrafted rules or stylistic rewrites, our method trains a language model to generate structured, self-reflective prompts by optimizing for image-level outcomes. The tailored reward models assesse the generated images in terms of human preference, semantic alignment, and visual composition, providing indirect supervision to refine prompt generation. Our approach enables end-to-end training without human-annotated data. Experiments on GenEval and T2I-Compbench show that RePrompt significantly boosts spatial layout fidelity and compositional generalization across diverse T2I backbones, establishing new state-of-the-art results. 
  Code is available at: 
\url{https://github.com/microsoft/DKI_LLM/tree/main/RePrompt}.
\end{abstract}

\section{Introduction}
Text-to-image (T2I) generation has made rapid progress with the rise of large-scale generative models~\cite{flux2024,esser2024scaling,podell2023sdxl,chen2024pixart}, yet a persistent challenge remains: users typically provide concise and under-specified prompts, which often result in images that fail to reflect the intended semantics or visually coherent compositions. Generated outputs may misrepresent object counts, overlook spatial relations, or violate real-world plausibility. This misalignment arises from the gap between natural language descriptions and the structured visual logic required for faithful image generation~\cite{yang2024mastering}.

Previous work on prompt enhancement in T2I can be divided into two main approaches. The first approach focuses on iterative refinement: an image is generated from an initial prompt, and subsequent feedback, derived either from human preference models or from automated scoring systems, is used to improve the prompt or intermediate representations over multiple rounds~\cite{yang2024idea2img,wu2024self,wang2024genartist,guo2025can}. Although this approach can progressively improve image quality, it suffers from high latency and computational overhead due to repeated image generation, and it rarely incorporates explicit scene semantics or compositional reasoning. The second approach enriches prompts in a single pass by leveraging large language models (LLMs) to inject additional detail and context~\cite{betker2023improving,hao2023optimizing}. While these methods produce linguistically fluent and expressive descriptions, they frequently generate prompts that produce images with semantically inconsistent or visually implausible content, such as conflicting object placements or unrealistic interactions, because the underlying LLMs lack grounding in physical reality and do not incorporate feedback from downstream visual tasks. As a result, they frequently hallucinate content or miss critical spatial and attribute-level relationships (see Figure~\ref{fig:intro}).

\begin{figure}[t]
  \centering
\includegraphics[width=0.97\linewidth]{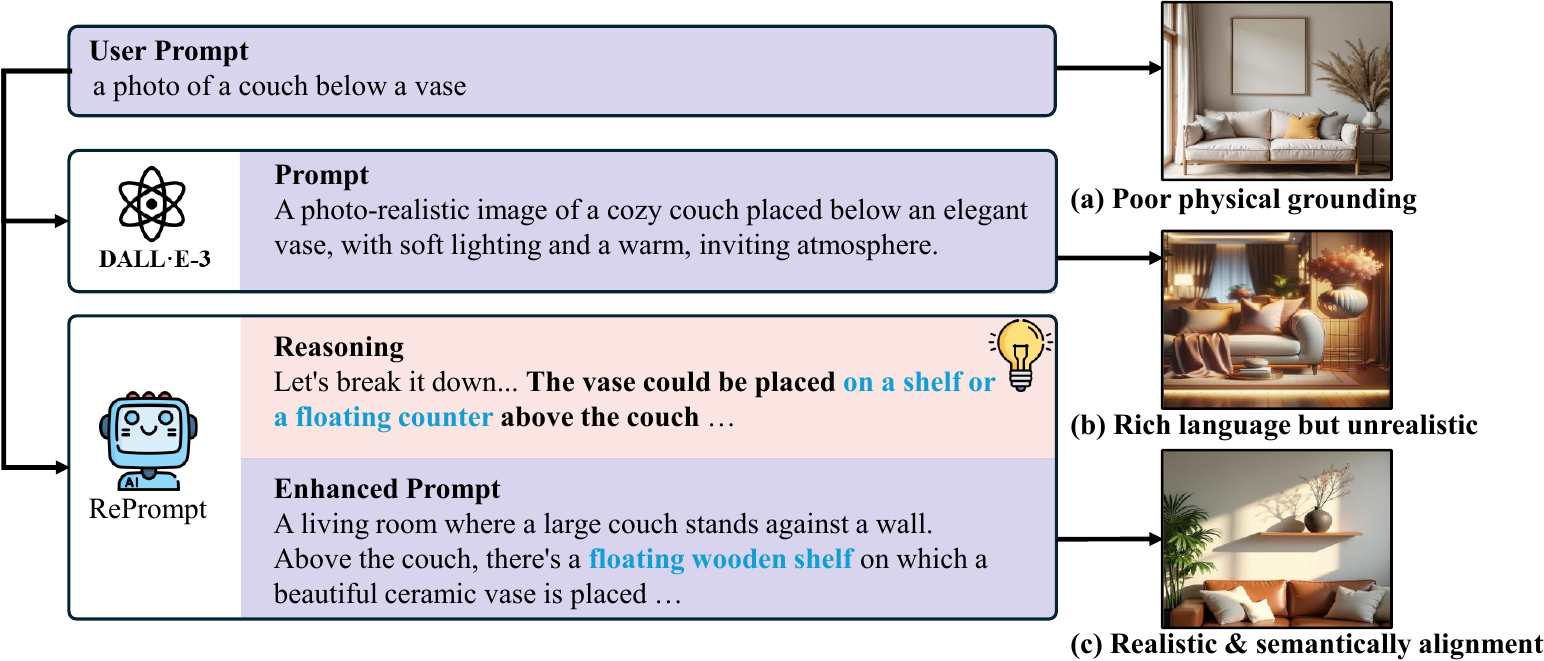}
  \caption{
    Given the user prompt "\textit{a photo of a couch below a vase}", existing models like DELL-E3 generate rich language descriptions but often produce unrealistic or physically implausible compositions. In contrast, our RePrompt performs explicit chain-of-thought reasoning to resolve spatial relations, resulting in enhanced prompts that guide text-to-image models towards realistic and semantically aligned generations. 
    }
  \label{fig:intro}
  \vspace{-0.5cm}
\end{figure}
In contrast, we propose RePrompt, a reasoning-augmented prompt refinement framework trained via reinforcement learning. Rather than relying on stylistic rewriting or black-box completions, RePrompt trains a language model to generate structured and semantically grounded prompts through self-reflection and step-by-step decomposition. Motivated by recent advances in reasoning-augmented language models~\cite{guo2025deepseek,trung2024reft,team2025kimi,jaech2024openai}, RePrompt enables the model to internally simulate the visual implications of a prompt—much like how humans mentally visualize a scene before drawing. This structured, logic-driven process anticipates potential errors (e.g., conflicting object positions, missing entities, or spatial incoherence) during prompt construction, thereby reducing the need for multiple rounds of image generation. 

A core component of RePrompt is a \textbf{T2I RePrompt Reward Model} tailored for text-to-image generation. Instead of relying on pre-labeled reasoning traces or handcrafted prompt templates, RePrompt learns from downstream visual feedback by optimizing prompt generation through reinforcement learning. To capture the multifaceted nature of image quality, we design a ensemble reward that evaluates generated images along three dimentions: human preference, visual realism, and semantic alignment with the input.
By learning from diverse and grounded feedback signals, the model develops a more robust reasoning strategy that transfers across prompt types, scene structures, and T2I backbones, enabling stronger performance on unseen inputs without overfitting to specific linguistic patterns.

Experiments on GenEval~\cite{ghosh2023geneval} and T2I-Compbench~\cite{huang2023t2i} demonstrate that RePrompt significantly improves compositional accuracy, semantic fidelity, and spatial coherence. Notably, on the GenEval benchmark, RePrompt surpasses Qwen2.5 3B-enhanced baselines by +77.1\% (FLUX~\cite{flux2024}), +78.8\% (SD3~\cite{esser2024scaling}) and +122.2\% (Pixart-$\Sigma$~\cite{chen2024pixart}) in the position category, highlighting its superior capability in grounding spatial relations. Furthermore, RePrompt achieves the best overall accuracy (0.76) among all evaluated methods while maintaining an order-of-magnitude lower latency than optimization-heavy baselines like Idea2Img (30s vs. 140s per image), offering a scalable and inference-efficient solution. These findings validate the effectiveness of explicit reasoning in prompt construction for closing the semantic-visual alignment gap in text-to-image generation, without relying on larger language models or expensive optimization at inference time.

\begin{figure}[t]
  \centering
\includegraphics[width=0.99\linewidth]{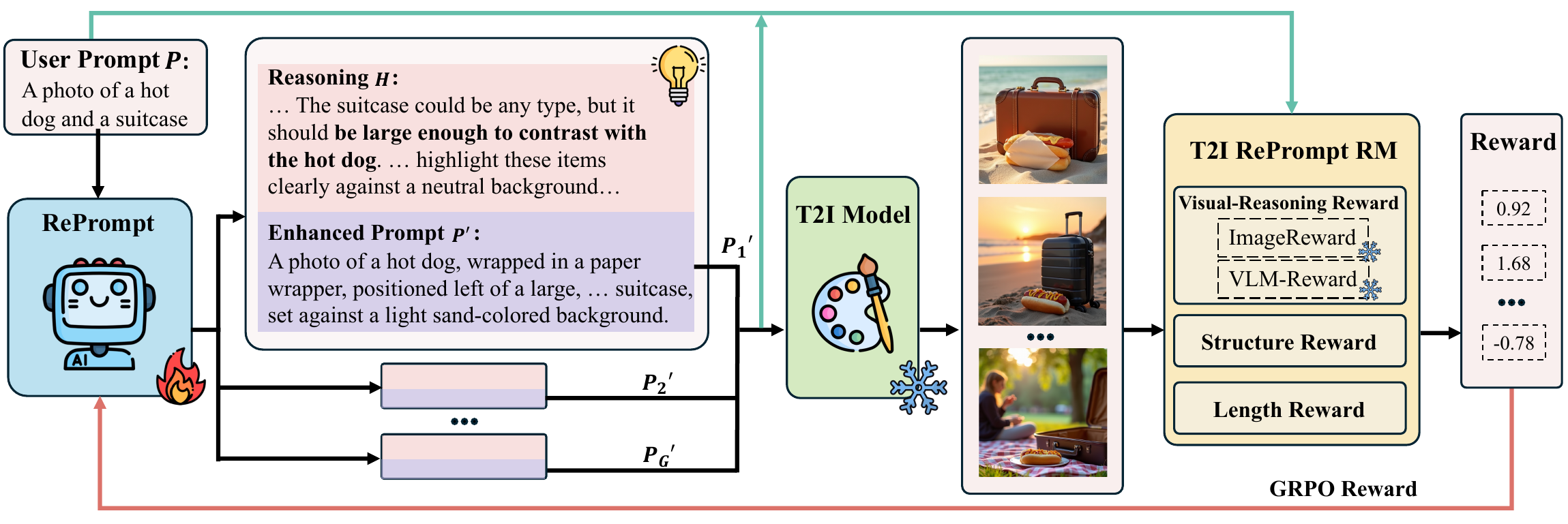}
  \caption{Overview of the proposed RePrompt. For each input prompt, RePrompt generates multiple reasoning trace and enhanced prompt pairs. The reasoning trace guides the model to produce more detailed, image-grounded prompts. These are used to synthesize candidate images via a T2I model, which are then scored by a reward model. Feedback is used to update RePrompt via GRPO.}
  \label{fig:method}
\end{figure}

\section{Related Work}
\paragraph{Text to Image Generation.}
Recently, large-scale diffusion models~\cite{flux2024,esser2024scaling,podell2023sdxl,chen2024pixart,xie2025sana,liu2024playground,betker2023improving,ma2025inference,rombach2022high,saharia2022photorealistic} have achieved impressive progress in generating high-resolution, photorealistic images from complex textual prompts. To enhance alignment between text and visuals, prior work has explored prompt engineering~\cite{hao2023optimizing,mo2024dynamic,yeh2024tipo,manas2024improving,yun2025learning,cao2023beautifulprompt,qin2024diffusiongpt,yang2024idea2img,wu2024self,wang2024genartist}, often relying on manual or heuristic strategies with limited generalization. We propose a reinforcement learning-based framework that automatically refines prompts through iterative reasoning, achieving better semantic alignment than static or rule-based methods.

\paragraph{Reinforcement Learning.}
Reinforcement learning (RL) has proven effective in scenarios where iterative feedback is essential for task optimization. In the realm of generative models, RL-based approaches~\cite{guo2025deepseek,wallace2024diffusion,yang2024dense,guo2024versat2i,gupta2025simple,lee2024parrot,zhao2024scores,nabati2024personalized,black2023training,liang2024rich,kirstain2023pick,lee2023aligning,shen2025vlm,zhang2024itercomp,wang2025simplear,jiang2025t2i,xue2025dancegrpo} have been employed to fine-tune outputs by maximizing a reward function that encapsulates desired attributes such as image realism, diversity, and semantic fidelity. In our framework, we adopt RL techniques to drive the automatic optimization of text prompts. By defining a multi-faceted reward that not only evaluates the visual quality of the generated image but also the interpretability and relevance of the prompt, our approach enables the model to learn an optimal prompt refinement strategy over successive iterations. Notably, T2I-R1~\cite{jiang2025t2i} is closely related to our work, but it targets Janus-Pro~\cite{chen2025janus}, a unified vision-language model. In contrast, our method trains an auxiliary LLM that generalizes across diverse text-to-image models, offering a more flexible and model-agnostic solution.

\paragraph{Reasoning in LLM.} 
Reasoning in LLMs~\cite{wei2022chain} improves complex task solving by decomposing problems into intermediate steps~\cite{feng2025video,huang2025vision,yang2025r1,zhang2025r1,huang2025boosting,yu2025perception,ma2025deepperception,li2025videochat,lu2025ui}. In multimodal generation, reasoning mechanisms~\cite{yang2024mastering,guo2025can,wang2025mint,zhang2025layercraft,sahili2024faircot,chen2024training} enhance prompt understanding and semantic alignment. Our method integrates reasoning with reinforcement learning, enabling step-wise prompt refinement that improves interpretability, text-image alignment, and image quality—offering a new perspective for prompt optimization in T2I generation.

\section{Method}
\label{sec:method}

We present \textbf{RePrompt}, a reasoning‐augmented reprompting framework for text‐to‐image (T2I) generation. RePrompt decouples prompt generation from image generation, i.e., training a language model to produce structured, semantically rich prompts, while keeping the T2I backbone fixed.  We optimize RePrompt via reinforcement learning (RL) to directly improve downstream image quality, compositional correctness, and usability.

\subsection{Framework Overview}
\label{sec:overview}

RePrompt comprises three main modules (Figure~\ref{fig:method}):  
1) a \textbf{Prompting Policy} \(\pi_\theta\), which produces a \emph{reasoning trace} \(H\) and an enhanced prompt \(P'\);  
2) a fixed \textbf{T2I Synthesizer} \(f_\phi\), which renders an image \(I\) from \(P'\);  
3) a \textbf{T2I RePrompt Reward Model} \(R_{\mathrm{total}}(I,P,P')\), which scores the image on realism, semantic alignment, and prompt structure.

Given an input prompt \(P\), the policy samples:
$y = (H, P') \sim \pi_\theta(P)$,
the synthesizer then generates:
$I = f_\phi(P')$. Since \(f_\phi\) is non‐differentiable, we formulate prompt generation as a single‐step Markov Decision Process (MDP):
\textbf{State}: the original prompt \(P\).  \textbf{Action}: sampling \(y=(H,P')\sim\pi_\theta(y\mid P)\).  \textbf{Transition}: deterministic mapping \(P'\mapsto I = f_\phi(P')\).  \textbf{Reward} \(r\): the reasoning reward \(r=R_{\mathrm{total}}(I,P,P')\).  \textbf{Objective}: maximize \(E_{P\sim\mathcal{D}}\big[E_{y\sim\pi_\theta(y\mid P)}[r]\big]\).

We train \(\pi_\theta\) via reinforcement learning (Group Relative Policy Optimization, GRPO~\cite{guo2025deepseek}) to improve downstream image quality.  By keeping the T2I model \(f_\phi\) fixed, RePrompt learns backbone‐specific reasoning strategies that enhance semantic fidelity and visual realism without requiring any manually annotated reasoning traces.

\subsection{T2I RePrompt Reward Model}
\label{sec:reward}

\begin{wrapfigure}{r}{0.45\textwidth}
    \centering
    \includegraphics[width=0.99\linewidth]{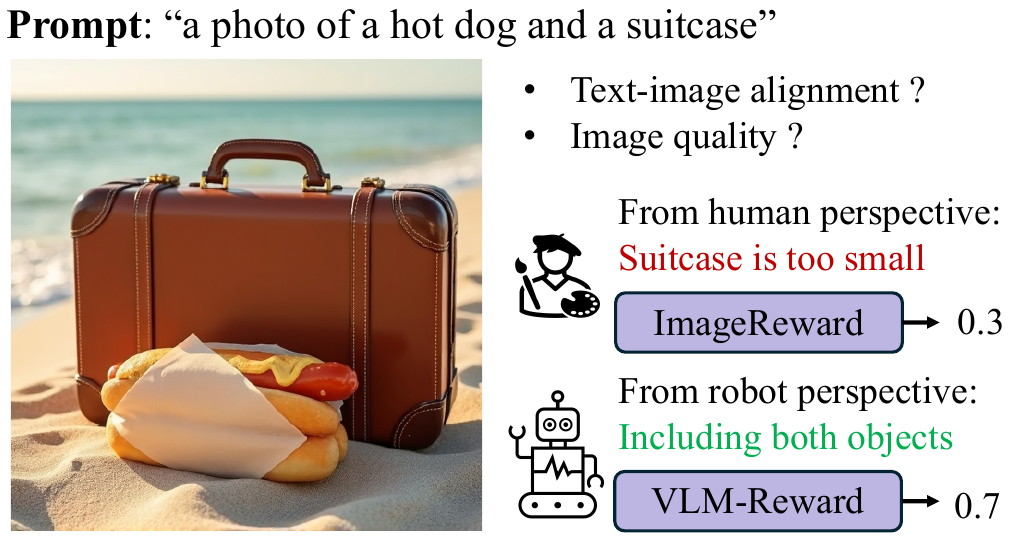}
    \caption{The Visual-Reasoning Reward.}
    \label{fig:reward}
    \vspace{-0.3cm}
\end{wrapfigure}

A central component of our framework is the \textbf{T2I RePrompt Reward Model}—an ensemble, image-grounded reward function specifically designed for the prompt refinement task in T2I generation. In contrast to generic reward functions used in natural language or vision tasks, our reward model is co-developed with the objective of enhancing reasoning-driven prompt construction. It explicitly evaluates whether a generated prompt yields an image that is \emph{realistic}, \emph{semantically faithful to the user intent}, and \emph{compositionally coherent}.

This reward framework is critical to training RePrompt effectively and serves three key goals:  
(1) \textbf{Stable optimization:} Each component provides dense and structured feedback, mitigating the challenges of sparse or noisy reward signals during early-stage learning.  
(2) \textbf{Multi-faceted supervision:} The reward captures complementary aspects of T2I quality, including human preference, visual quality, and semantic alignment, ensuring holistic prompt improvement.  
(3) \textbf{Cross-model generalization:} Because the reward depends only on the prompt-image output pair and not on any specific T2I architecture, it generalizes naturally across different generation backbones and unseen prompt distributions. Together, these properties enable our reward model to not only guide the training of the reprompting policy but also ensure broad applicability and stable learning across varying settings in text-to-image generation.

\paragraph{Visual‐Reasoning Reward (\(R_{\mathrm{vis}}\)).}
This component (as shown in Figure~\ref{fig:reward}) captures both user-aligned preferences and semantic correctness at the image level:
\begin{equation}
    R_{\mathrm{vis}} = \alpha \,R_{\mathrm{pref}}^{\text{IMG}} + \gamma \,R_{\mathrm{sem}}^{\text{VLM}}.
\end{equation}
Here, \(R_{\mathrm{pref}}^{\text{IMG}}\) is derived from ImageReward~\cite{xu2023imagereward}, which used to evaluate whether generated images align with human preferences. \(R_{\mathrm{sem}}^{\text{VLM}}\) is obtained from VLM-Reward~\cite{achiam2023gpt}, which evaluates semantic consistency and visual quality using a vision-language model. The weights \(\alpha\) and \(\gamma\) allow us to control the trade-off between perceived image quality and factual alignment.

\paragraph{Structure Reward (\(R_{\mathrm{struc}}\)).}
To ensure that the generated output maintains a clear reasoning-to-prompt format, we enforce a structured syntax:
\begin{quote}
\texttt{<reason>…</reason><prompt>…</prompt>}
\end{quote}
We apply a binary reward:
\begin{equation}
    R_{\mathrm{struc}} = 
\begin{cases}
+1, & \text{if format is correct} \\
-1, & \text{otherwise}
\end{cases}
\end{equation}

This reward encourages models to adhere to a standardized output layout, simplifying downstream parsing and ensuring the interpretability of reasoning traces.

\paragraph{Length Reward (\(R_{\mathrm{len}}\)).}
To ensure compatibility with real-world T2I models such as SDXL and FLUX.1, which impose token-length limits, we apply a constraint on prompt length:
\begin{equation}
    R_{\mathrm{len}} =
\begin{cases}
+1, & L_{\min} \leq L \leq L_{\max}, \\
-1, & \text{otherwise}
\end{cases}
\end{equation}
where \(L\) is the token length of \(P'\), and \([L_{\min}, L_{\max}]\) is empirically set to 15 and 77 tokens. This ensures the prompt is both concise and informative.

\paragraph{Ensemble Reward and Optimization.} 
All reward components are normalized to unit variance and summed:
\begin{equation}
    R_{\mathrm{total}} = R_{\mathrm{vis}} + R_{\mathrm{struc}} + R_{\mathrm{len}}.
\end{equation}
The RePrompt policy \(\pi_\theta\) is trained to maximize expected downstream reward:
\begin{equation}
    \theta^{\ast} = \arg\max_\theta \;
    \mathbb{E}_{P \sim \mathcal{D},\, y \sim \pi_\theta(y \mid P)}\left[ R_{\mathrm{total}}(I, P, P') \right],
\end{equation}
where \(I = f_\phi(P')\) and \(f_\phi\) is a fixed text-to-image generator.

\paragraph{Generalization and Flexibility.}
Since the reward depends only on the input-output behavior of the system, not the internals of the image model, our framework generalizes well across unseen prompts and new T2I backbones. This modularity also enables RePrompt to adapt to model-specific strengths and failure patterns, while maintaining stable and interpretable training signals throughout.

\subsection{Reprompt Optimization}
\label{sec:optimization}
Prompt generation for T2I involves a non‐differentiable, black‐box image renderer \(f_\phi\).  RL allows us to optimize \(\pi_\theta\) directly with respect to \emph{downstream image outcomes} rather than proxy losses on text.  Moreover, because RePrompt is \emph{individualized to each T2I backbone}, it can adapt its reasoning and prompt style to the specific strengths and limitations of a given model—improving generalization and image fidelity without retraining \(f_\phi\).

At each update step, for a given user prompt \(P\), we sample a set of \(G\) candidate outputs \(\{y_i\}_{i=1}^G\sim\pi_{\theta_{\mathrm{old}}}(y\mid P)\).  Each candidate \(y_i=(H_i,P'_i)\) yields an image \(I_i=f_\phi(P'_i)\) and receives a scalar reward \(r_i = R_{\mathrm{total}}(I_i,P,P')\).  We then compute normalized advantages:
\begin{equation}
    A_i = \frac{r_i - \mu_r}{\sigma_r}, 
\quad \mu_r = \frac{1}{G}\sum_{j=1}^G r_j,\quad
\sigma_r = \sqrt{\frac{1}{G}\sum_{j=1}^G (r_j - \mu_r)^2}.
\end{equation}

We use Group Relative Policy Optimization (GRPO)~\cite{guo2025deepseek} as a practical and stable update method for training the reprompting policy based on group-wise reward comparisons. The objective is defined as:
\begin{align}
\mathcal{J}_{\mathrm{GRPO}}(\theta) 
= E_{P,\{y_i\}\sim\pi_{\theta_{\mathrm{old}}}} \Bigg[\,
\frac{1}{G}\sum_{i=1}^G 
&\min\!\Big(\rho_i\,A_i,\;\mathrm{clip}(\rho_i,1-\varepsilon,1+\varepsilon)\,A_i\Big) \nonumber\\
&- \beta_{\mathrm{KL}}\;KL\big(\pi_\theta(y\mid P)\,\|\,\pi_{\mathrm{ref}}(y\mid P)\big)\Bigg],
\label{eq:grpo}
\end{align}
where 
\(\rho_i = \frac{\pi_\theta(y_i\mid P)}{\pi_{\theta_{\mathrm{old}}}(y_i\mid P)}\),  
\(\varepsilon\) and \(\beta_{\mathrm{KL}}\) are clipping and penalty coefficients,  
\(\pi_{\mathrm{ref}}\) is a reference policy (e.g., the initial or distillation policy). By fixing \(f_\phi\) and optimizing only \(\pi_\theta\), RePrompt can be applied \emph{universally} to any off-the-shelf T2I model, learning backbone, specific reasoning and prompt strategies without retraining the image generator.  
We further validate RePrompt with a variance-reduction analysis, showing that structured reasoning reduces reward uncertainty and lowers the sample complexity for accurate estimation. This leads to faster and more stable GRPO training. Full analysis and proof are provided in Appendix~\ref{app:a}.

\section{Experiments}


\subsection{Settings}
\label{sec:detail}
\paragraph{Implementation Details.}
We use Qwen2.5-3B~\cite{yang2024qwen2} as our base language model. For the text-to-image model used for training, we use the FLUX.1-dev~\cite{flux2024} model, which generates images at a resolution of 512×512 pixels. 
Our model is trained using the TRL~\footnote{https://github.com/huggingface/trl} reinforcement learning framework for 3 epochs, with 4 outputs generated per instance, the weight of ImageReward~\cite{xu2023imagereward} and VLM-Reward are both 0.5. The VLM used for computing VLM-Reward is GPT-4V. All experiments were conducted on 8 NVIDIA A100 (80GB) GPUs, and the entire training process required about 6 hours. More details are in the Appendix~\ref{app:b}.

\paragraph{Training Data.}
Inspired by the prompt construction strategy in GenEval~\cite{ghosh2023geneval}, we adapt six object-centric templates to a newly curated list of 288 common daily objects generated via GPT-4~\cite{achiam2023gpt}. This results in a training corpus of 9,000 prompts, carefully filtered to avoid overlap with the GenEval. We use 8,000 prompts to fine-tune our RePrompt via supervised learning, and 1,000 prompts for reinforcement learning.

\paragraph{Evaluation Setup.}
To assess the impact of RePrompt, we evaluate on two benchmarks: GenEval~\cite{ghosh2023geneval} and T2I-Compbench~\cite{huang2023t2i}. GenEval focuses on instance-level alignment with user intent using concise prompts, while T2I-Compbench measures compositional generation under complex scenarios involving multiple objects, attributes, and spatial relations.

\begin{table}[t]
  \caption{Evaluation of text-to-image generation on the GenEval benchmark. Our method consistently outperforms strong baselines, achieving the best overall scores. Notably, our approach shows substantial gains in spatial position understanding over Qwen2.5 3B-enhanced baselines, demonstrating its superior capability in grounding spatial relations. }
  \label{tab:geneval}
  \centering
  \small
  \setlength{\tabcolsep}{4pt} 
  \begin{tabular}{l|cccccc|c}
    \toprule
    Method  & Single & Two & Counting & Colors & Position & Attribute & Overall~$\uparrow$ \\
           &            object & object &          &        &          & binding   &  \\
    \midrule
    FLUX~\cite{flux2024}          & 0.99 & 0.79 & 0.75 & 0.78 & 0.18 & 0.45  & 0.66 \\
    +GPT4 & 0.99 & 0.79 & 0.68 & 0.84 & 0.51 & 0.52 & 0.72 \\
    +Deepseek-r1 & 1.00 & 0.81 & 0.56 & 0.78 & 0.47 & 0.43 & 0.67\\
    +Qwen2.5 7B & 0.99 & 0.83 & 0.62 & 0.84 & 0.36 & 0.51 & 0.69 \\
    +Qwen2.5 3B      & 0.99 & 0.84 & 0.63 & 0.81 & 0.35 & 0.48 & 0.68 \\
    +Ours (train w/ FLUX) & 0.98 & \textbf{0.87} & \textbf{0.77} & \textbf{0.85} & \textbf{0.62} & \textbf{0.49} & \textbf{0.76} \\
    \hdashline
    \textit{Improvement over Qwen2.5 3B} & -1.0\% & +3.6\% & +22.2\% & +4.9\% & \textbf{+77.1\%} & +2.1\% & +11.8\% \\
    
    \midrule
    SD3~\cite{esser2024scaling}           & 1.00 & 0.85 & 0.62 & 0.88 & 0.22 & 0.58 & 0.69 \\
    +GPT4 & 1.00 & 0.84 & 0.51 & 0.85 & 0.48 & 0.54 & 0.70 \\
    +Deepseek-r1 & 0.99 & 0.82 & 0.53 & 0.80 & 0.44 & 0.46 & 0.67 \\
    +Qwen2.5 7B & 1.00 & 0.82 & 0.49 & 0.85 & 0.34 & 0.58 & 0.68 \\
    +Qwen2.5 3B     & 1.00 & 0.86 & 0.53 & 0.84 & 0.33 & 0.55 & 0.68 \\
    +Ours (train w/ FLUX)    & 0.99 & \textbf{0.86} & 0.60 & 0.86 & \textbf{0.59} & \textbf{0.60} & \textbf{0.75} \\
    \hdashline
    \textit{Improvement over Qwen2.5 3B} & -1.0\% & 0.0\% & +13.2\% & +2.4\% & \textbf{+78.8\%} & +9.1\% & +10.3\% \\
    \midrule
    Pixart-$\Sigma$~\cite{chen2024pixart} & 0.99 & 0.60 & 0.47 & 0.81 & 0.10 & 0.26 & 0.54 \\
    +GPT4 & 0.96 & 0.67 & 0.48 & 0.84 & 0.36 & 0.31 & 0.60 \\
    +Deepseek-r1 & 0.99 & 0.63 & 0.43 & 0.78 & 0.24 & 0.27 & 0.56 \\
    +Qwen2.5 7B & 0.96 & 0.67 & 0.43 & 0.83 & 0.20 & 0.32 & 0.57 \\
    +Qwen2.5 3B & 0.99 & 0.68 & 0.48 & 0.82 & 0.18 & 0.32 & 0.58 \\
    +Ours (train w/ FLUX) & 0.98 & 0.64 & \textbf{0.56} & \textbf{0.81} & \textbf{0.40} & \textbf{0.35} & \textbf{0.62} \\
    \hdashline
    \textit{Improvement over Qwen2.5 3B} & -1.0\% & -5.9\% & +16.7\% & -1.2\% & \textbf{+122.2\%} & +9.4\% & +6.9\% \\
    \bottomrule
  \end{tabular}
\end{table}

\begin{table}[t]
  \caption{Evaluation of text-to-image generation on the T2I-Compbench. We report the baseline results, their variants enhanced with Qwen2.5 3B, and our method trained with FLUX. Our approach consistently improves performance across most aspects, particularly in Spatial compositions. }
  \label{tab:comp}
  \centering
  \small
  \begin{tabular}{lcccccc}
    \toprule
    Method     & Color & Shape & Texture & Spatial & Numeracy & Complex\\
    \midrule
    FLUX &  0.7223 & 0.4796 & 0.6522 & 0.2494 & 0.6101 & 0.3616\\
    +Qwen2.5 3B     & 0.7149 & 0.5103 & 0.6012 & 0.2579 & 0.5982 & 0.3325\\
    +Ours~(train w/ FLUX)   &  \textbf{0.7501} & \textbf{0.5276} & 0.6515 & \textbf{0.3301} & \textbf{0.6499} & \textbf{0.3721} \\
    \midrule
    SD3 & 0.7941 & 0.5812 & 0.7224 & 0.2815 &  0.5871 & 0.3714 \\
    +Qwen2.5 3B & 0.7227 & 0.5478 & 0.6581 & 0.2549  & 0.5934 & 0.3307\\
    +Ours~(train w/ FLUX) & 0.7866 & \textbf{0.5891} & 0.7184 & \textbf{0.3315} & \textbf{0.6289} & \textbf{0.3744} \\
    \midrule
    Pixart-$\Sigma$ & 0.5682 & 0.4717 & 0.5622 & 0.2497 & 0.5366 & 0.3655   \\
    +Qwen2.5 3B & 0.6618 & 0.4814 & 0.5662 & 0.2481 & 0.5443 &  0.3335 \\
    +Ours~(train w/ FLUX) & \textbf{0.6665} & \textbf{0.5011} & \textbf{0.6190} & \textbf{0.2913} & \textbf{0.5716} & \textbf{0.3680}   \\ 
    \bottomrule
  \end{tabular}
  \small
\end{table}

\subsection{Comparision}
Table~\ref{tab:geneval} presents the performance comparison across various text-to-image generation models evaluated on the GenEval benchmark, which assesses six fine-grained composition capabilities: single-object, two-object, counting, color, spatial position, and attribute binding.

Notably, our method demonstrates exceptional gains in spatial layout understanding (Position). For instance, when built upon FLUX~\cite{flux2024} , our approach achieves a 0.62 score on Position, representing a +77.1\% relative improvement over the Qwen2.5 3B baseline. Similarly, for SD3~\cite{esser2024scaling}, we observe a +78.8\% gain, and for Pixart-$\Sigma$~\cite{chen2024pixart}, the relative improvement reaches an impressive +122.2\%. This strong enhancement highlights the strength of our compositional training strategy in explicitly grounding spatial relations between objects. Beyond the Position metric, our method also achieves substantial improvements in Counting (+22.2\% for FLUX, +13.2\% for SD3, +16.7\% for Pixart-$\Sigma$) and moderate gains in attribute binding (+2.1\% to +9.4\%). As a result, we observe consistent boosts in overall GenEval scores: +11.8\% (FLUX), +10.3\% (SD3), and +6.9\% (Pixart-$\Sigma$). These results verify the effectiveness of our reinforcement learning–based reprompting strategy, which enables the language model to iteratively reason about visual composition and generate more precise, image-aligned prompts, and brings state-of-the-art advances in spatial understanding.


\begin{table}[t]
  \caption{Ablation study of SFT and RL on the GenEval benchmark.}
  \label{tab:abl-sft}
  \centering
  \small
  \setlength{\tabcolsep}{4pt} 
  \begin{tabular}{l|cccccc|c}
    \toprule
    Method  & Single & Two & Counting & Colors & Position & Attribute & Overall~$\uparrow$ \\
           &            object & object &          &        &          & binding   &  \\
    \midrule
    FLUX~\cite{flux2024}          & 0.99 & 0.79 & 0.75 & 0.78 & 0.18 & 0.45  & 0.66 \\
    +Qwen2.5 3B      & 0.99 & 0.84 & 0.63 & 0.81 & 0.35 & 0.48 & 0.68 \\
    \midrule
    w/ SFT & 0.99 & 0.83 & 0.64 & 0.81 & 0.43 & 0.44 & 0.69 \\
    w/ RL & 0.98 & 0.83 & 0.71 & \textbf{0.87} & 0.41 & \textbf{0.53} & 0.72 \\
    w/ SFT + RL & 0.98 & \textbf{0.87} & \textbf{0.77} & 0.85 & \textbf{0.62} & 0.49 & \textbf{0.76} \\
    
    \bottomrule
  \end{tabular}
\end{table}

\begin{table}[t]
\centering
\caption{Quantitative comparison between RePrompt and other method on image generation accuracy and latency with the subset of Geneval. All latency is measured on a single NVIDIA A100 GPU.}
\label{tab:comparison}
\small
\begin{tabular}{lcc}
\toprule
\textbf{Method} & \textbf{Accuracy} ↑ & \textbf{Latency (s / img)} ↓ \\
\midrule
FLUX & 0.65 & 20\\
Show-o~\cite{xie2024show} & 0.55 & 3\\
\midrule
Idea2Img~\cite{yang2024idea2img} (w/ FLUX) & 0.69 & 140 \\
PARM++~\cite{guo2025can} (w/ Show-o)       & 0.72                     & 110                   \\
\textbf{RePrompt} (w/ FLUX) & \textbf{0.76}         & \textbf{30}          \\
\bottomrule
\end{tabular}
\small
\end{table}

\subsection{Generalization Performance on T2I-Compbench}
We further assess the robustness of our method on the T2I-Compbench benchmark, which tests compositional generalization across six dimensions: color, shape, texture, spatial reasoning, numeracy, and complex attribute combinations. As shown in Table~\ref{tab:comp}, RePrompt consistently improves performance across all evaluated backbones. In particular, it significantly boosts spatial compositional scores (e.g., from 0.2494 to 0.3301 on FLUX and from 0.2815 to 0.3315 on SD3) and enhances numeracy understanding, two long-standing challenges in text-to-image generation.
Moreover, our method outperforms stronger LLM-enhanced baselines, for instance, on Pixart-$\Sigma$, RePrompt achieves notable gains in both texture and spatial dimensions. These results demonstrate that our approach generalizes well across diverse models and compositional skills, validating its effectiveness as a versatile and plug-and-play enhancement for real-world T2I systems.


\subsection{Ablation Study}
\paragraph{Ablation on SFT and RL.}
Table~\ref{tab:abl-sft} shows the effect of supervised fine-tuning (SFT) and reinforcement learning (RL) on GenEval. SFT brings modest gains (+0.01 overall), mainly improving positional understanding (0.35→0.43), suggesting it helps inject object-attribute knowledge but struggles with complex reasoning. RL yields larger boosts (+0.04 overall), as it directly optimizes visual correctness. Combining SFT and RL achieves the best results (0.76 overall), with strong improvements in spatial reasoning (0.62) and counting (0.77). These results confirm that SFT offers useful priors, while RL is key for compositional robustness.

\begin{figure}[t]
  \centering
    \includegraphics[width=0.93\linewidth]{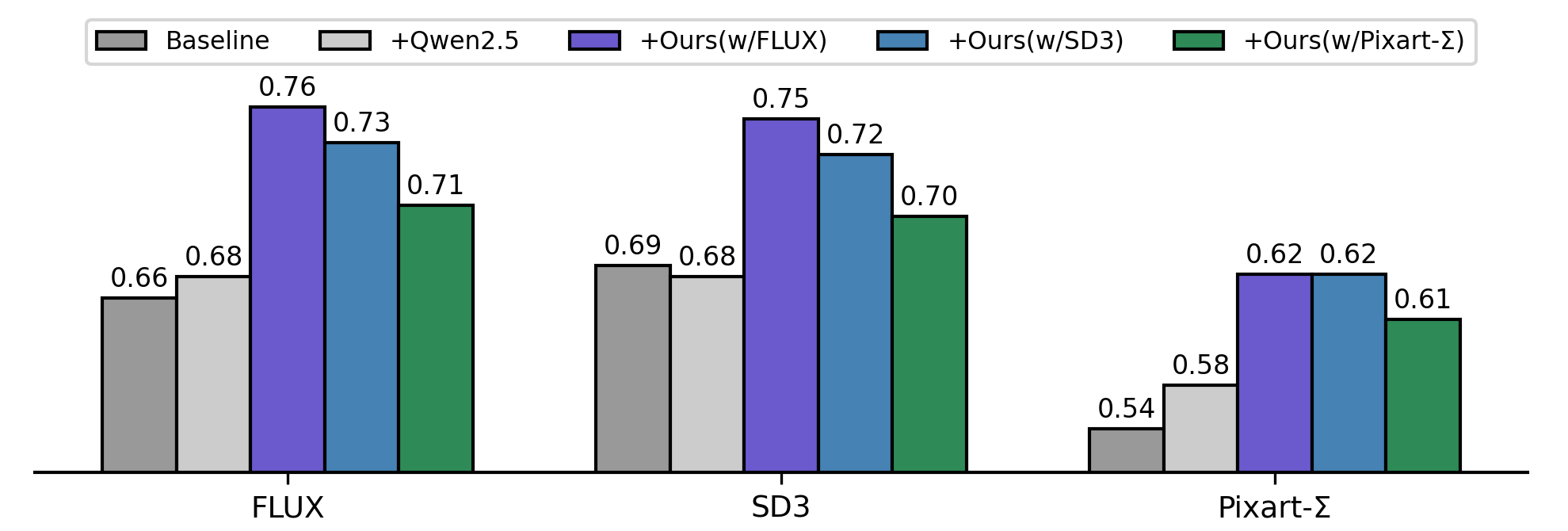}
  \caption{Impact of our method across different base T2I models on the GenEval benchmark. Our method consistently improves the compositional understanding across all base models.}
  \label{fig:abl-t2i}
\end{figure}


\begin{figure}[t]
  \centering
    \includegraphics[width=0.96\linewidth]{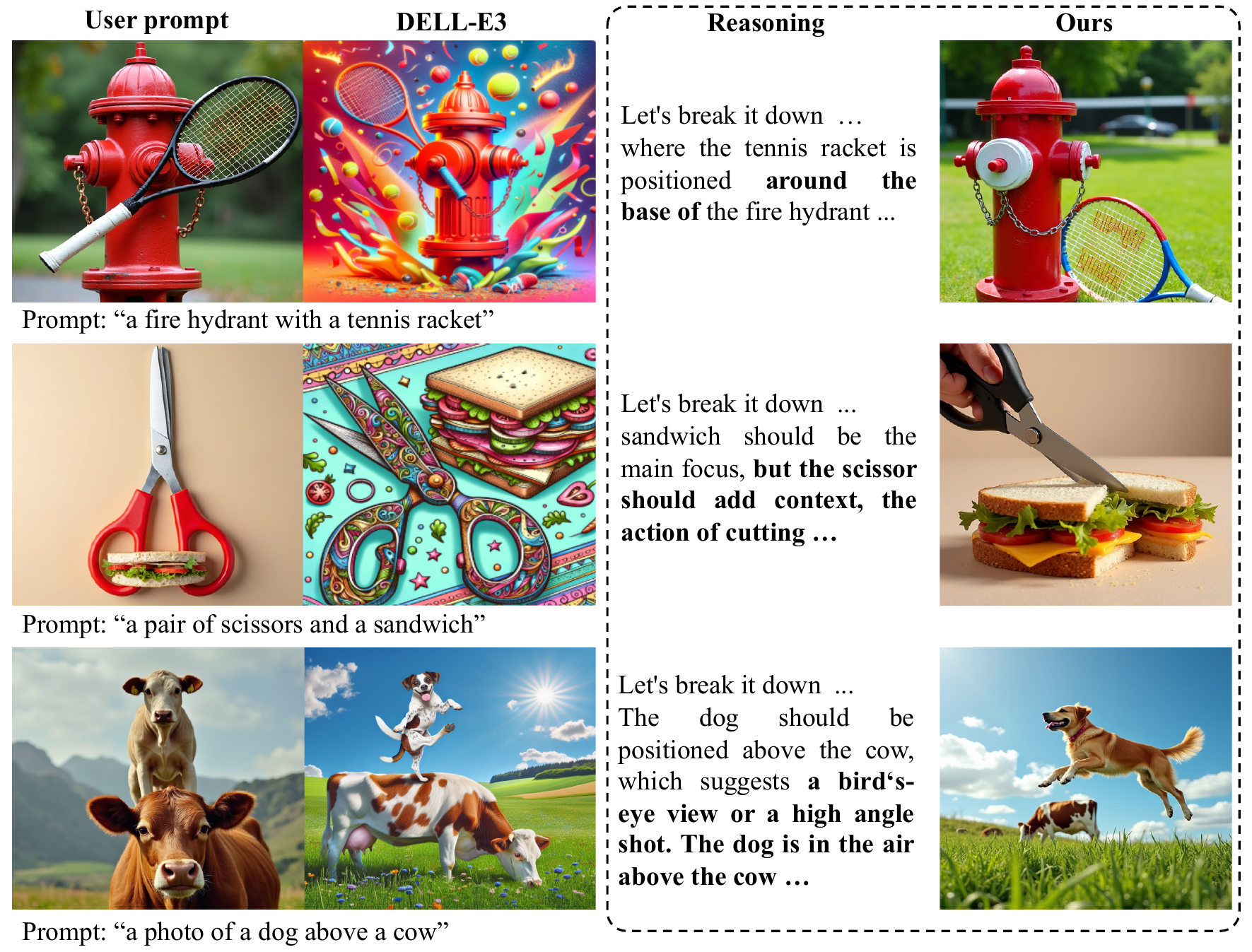}
  \caption{Qualitative results on compositional prompts. Compared to vanilla T2I models, our RePrompt improves spatial layout and object relations by generating enhanced prompts with explicit reasoning, leading to more faithful compositions.}
  \label{fig:vis}
  \vspace{-0.4cm}
\end{figure}

\paragraph{Comparison on Accuracy and Latency.}
Table~\ref{tab:comparison} benchmarks our RePrompt against prior methods in terms of both generation accuracy (GenEval overall) and inference latency. Among base models, FLUX.1 achieves higher accuracy (0.65) compared to Show-o (0.55) but suffers from a longer latency (20s vs. 3s per image), reflecting a trade-off between quality and speed. When incorporating advanced prompting techniques, Idea2Img (w/ FLUX.1) improves accuracy to 0.69 but at the cost of a significant latency increase (140s per image), while PARM++ (w/ Show-o) achieves 0.72 accuracy with 110s latency. 
In contrast, our RePrompt achieves the best accuracy (0.76) while maintaining a much lower latency (30s). This demonstrates the effectiveness of our design in enabling precise visual grounding without compromising efficiency, making it more practical for real-world deployment.

\paragraph{Ablation Study on Trained T2I Models.} 
Figure~\ref{fig:abl-t2i} illustrates the generalization ability of RePrompt across three diverse T2I backbones: FLUX, SD3, and PixArt-$\Sigma$. Our method consistently boosts compositional scores on GenEval—improving FLUX from 0.66 to 0.76, SD3 from 0.69 to 0.75, and PixArt-$\Sigma$ from 0.54 to 0.62—outperforming both baselines and the +Qwen2.5 variants. These results demonstrate RePrompt’s robustness and model-agnostic design, working effectively across both large (e.g., SD3) and lightweight (e.g., PixArt-$\Sigma$) backbones. Notably, we observe that stronger T2I models used during training lead to better generalization at inference, likely due to richer reasoning patterns induced during learning.


\paragraph{Qualitative Comparison.}
Figure~\ref{fig:vis} illustrates the qualitative advantage of RePrompt over existing T2I models. Baseline models often generate images with incorrect spatial relations or hallucinated objects. For example, when prompted with “a fire hydrant with a tennis racket,” DALL-E 3 produces unrealistic, stylistic blends where the objects are merged. In contrast, RePrompt accurately grounds the tennis racket "around the base" of the fire hydrant, respecting the intended composition. Similarly, for “a photo of a dog above a cow,” our method correctly depicts the dog in the air above the cow with "a high angle shot", aligning with the prompt semantics. These results highlight RePrompt’s ability to mitigate typical failures in spatial arrangement and object interaction by generating prompts with explicit compositional cues. More cases are in the Appendix~\ref{app:d}.

\paragraph{More Ablation Study.} 
We present additional ablation studies on reasoning, training dynamics, and reward functions in the Appendix~\ref{app:c} to further validate the effectiveness of our reasoning design and reward selection, as well as the stability of the training process.

\section{Limitations}
\label{sec:limit}

While our method consistently improves compositional quality across standard T2I benchmarks, several limitations warrant future exploration. First, the performance gains on certain fine-grained tasks—such as numeracy and object counting—remain modest, suggesting potential for further enhancement in precise quantitative reasoning. Second, our approach requires fine-tuning with compositional supervision, which may introduce additional computational cost and reliance on structured training signals. However, this design is consistent with common RLHF setups and does not limit practical deployment. Third, the effectiveness of RePrompt depends on the quality of the reward model; while we demonstrate robustness across standard evaluators, improvements in reward fidelity could further amplify performance. 

\section{Conclusion}
In this work, we propose a simple yet effective method to enhance compositional understanding in text-to-image (T2I) generation models. By injecting chain-of-thought (CoT) reasoning into the prompt construction pipeline and pairing each CoT with an enhanced prompt, our approach improves the alignment between textual descriptions and generated images. Extensive experiments on the GenEval benchmark demonstrate that our method consistently improves performance across various T2I backbones, including FLUX, SD3, and Pixart-$\Sigma$, without requiring additional model retraining. These results highlight the generalizability and plug-and-play nature of our method. We believe our approach offers a new perspective for improving controllability and compositional fidelity in generative models.



\bibliographystyle{plain}
\bibliography{neurips_2025}

\newpage
\appendix



\section{Variance Reduction via Structured Reasoning}
\label{app:a}

In this appendix, we provide a full theoretical analysis showing that conditioning prompt generation on an explicit reasoning trace \(H\) strictly reduces the variance of the downstream reward estimator.  This reduction in variance directly leads to improved sample efficiency for reinforcement learning.

\subsection{Setup and Notation}

Let:
\begin{itemize}
  \item \(\mathcal{P}\) be the space of bare prompts \(P'\).
  \item \(\mathcal{H}\) be the space of reasoning traces \(H\).
  \item \(r: \mathcal{P} \to [0,R_{\max}]\) be the reward random variable obtained by sampling \(P'\sim \pi(P)\) and generating an image \(I=f_\phi(P')\).
  \item \(r_H: \mathcal{H}\times\mathcal{P}\to [0,R_{\max}]\) be the reward when first sampling \(H\sim \pi_H(P)\), then \(P'\sim \pi(P\,|\,H)\), and finally \(I=f_\phi(P')\).
  \item All expectations and variances are taken over the joint sampling of \(H\) and \(P'\).
\end{itemize}

\subsection{Law of Total Variance}

By the law of total variance,
$$\mathrm{Var}\bigl[r(P')\bigr]
= E_H\bigl[\mathrm{Var}[r(P')\mid H]\bigr]
  + \mathrm{Var}_H\bigl[E\bigl[r(P')\mid H\bigr]\bigr.]$$
Since variances are nonnegative, we immediately have:

\begin{theorem}[Variance Reduction]
\label{thm:variance_reduction}
$$
\mathrm{Var}\bigl[r(H,P')\bigr]
   = E_H\bigl[\mathrm{Var}\bigl[r\mid H\bigr]\bigr]
   \;\le\;
   \mathrm{Var}\bigl[r(P')\bigr].
$$
\end{theorem}

\begin{proof}
By definition,
$$
\mathrm{Var}\bigl[r(H,P')\bigr]
= E_H[\mathrm{Var}[r\mid H]]
$$
and since
$$
\mathrm{Var}[r(P')]
= E_H[\mathrm{Var}[r\mid H]] + \mathrm{Var}_H[E[r\mid H]],
$$
the nonnegativity of $\mathrm{Var}_H[E[r\mid H]]$ yields the result.
\end{proof}

\subsection{Sample Complexity Improvement}

Variance directly impacts the number of samples required to estimate the expected reward within a given accuracy.  Consider estimating the expected reward \(\mu = E[r]\) by drawing \(N\) independent samples \(\{r_i\}\).  By Chebyshev’s inequality,
$$
\Pr\Bigl(|\hat\mu - \mu|\ge \varepsilon\Bigr)
\;\le\;\frac{\mathrm{Var}[r]}{N\varepsilon^2}.
$$
Thus, to guarantee
$\Pr(|\hat\mu - \mu|\ge \varepsilon)\le \delta$,
we require
$$
N \;\ge\; \frac{\mathrm{Var}[r]}{\varepsilon^2\,\delta}.
$$
Applying Theorem~\ref{thm:variance_reduction}, using reasoning reduces the required sample size:
$$
N_{\text{reasoning}}
\;=\;
\frac{\mathrm{Var}[r(H,P')]}{\varepsilon^2\delta}
\;\le\;
\frac{\mathrm{Var}[r(P')]}{\varepsilon^2\delta}
\;=\;
N_{\text{bare}}.
$$

\subsection{Discussion}

This analysis formalizes the intuition that introducing an intermediate reasoning trace \(H\) conditions the policy on structured, compositional information about the scene, thereby reducing uncertainty (variance) in the reward signal.  Empirically, this translates to fewer rollouts needed during GRPO training and more stable gradient estimates—accelerating convergence without additional data or annotations.

\section{Experiment Setting}
\label{app:b}
Table~\ref{tab:grpo_t2i_params} summarizes the key hyperparameters used in our experiments, including configurations for the GRPO optimization algorithm, the FLUX.1 text-to-image model, and the joint training process. For GRPO, we set the clipping range $\varepsilon$ to 0.2 and the KL penalty coefficient $\beta$ to 0.04 with a group size of 4. The FLUX.1 model operates at a resolution of $512 \times 512$, with 50 diffusion steps and a classifier-free guidance (CFG) scale of 3.5. In joint training, we balance ImageReward and VLM-Reward with equal weights (0.5 each), and constrain prompt lengths between 15 and 77 tokens. Training is conducted using 8 devices with a per-device batch size of 4, a learning rate of 2e-6, gradient accumulation steps of 2, and a total of 3 epochs.
\begin{table}[h]
\centering
\caption{Key Parameters for GRPO and T2I Model.}
\label{tab:grpo_t2i_params}
\begin{tabular}{lcc}
\toprule
\textbf{Config} & \textbf{Symbol} & \textbf{Value} \\ 
\midrule
\multicolumn{3}{c}{\textbf{GRPO Config}} \\ 
\midrule
Clipping range & $\varepsilon$ & 0.2 \\ 
KL penalty coefficient & $\beta$ & 0.04 \\ 
Group size & $G$ & 4 \\ 
\midrule

\multicolumn{3}{c}{\textbf{FLUX.1 Config}} \\ 
\midrule
Image resolution & $H \times W$ & $512 \times 512$ \\ 
Diffusion steps & $T$ & 50 \\ 
CFG scale & $\lambda_{\text{cfg}}$ & 3.5 \\ 
\midrule

\multicolumn{3}{c}{\textbf{Joint Training Parameters}} \\ 
\midrule
ImageReward weight & $\alpha$ & 0.5 \\
VLM-Reward weight & $\gamma$ & 0.5 \\
Min prompt’s length & $L_{\min}$ & 15 \\
Max prompt’s length & $L_{\max}$ & 77 \\
Device number & - & 8 \\
Per device batch size & $B$ & 4 \\ 
Learning rate & $lr$ & 2e-6 \\
Gradient Accumulation & - & 2 \\
Epoch & - & 3 \\
\bottomrule
\end{tabular}
\end{table}

\section{Ablation Study}
\label{app:c}
\paragraph{Ablation Study on Reasoning.}
Table~\ref{tab:abl-reason} presents an ablation study on the GenEval benchmark to assess the impact of reasoning in our reinforcement learning framework. The FLUX baseline, integrating a large language model (+Qwen2.5 3B) brings modest gains across most categories, raising the overall score to 0.68. Applying RL without reasoning achieves a similar overall improvement (0.68), suggesting that reward-driven optimization alone contributes to better alignment. However, incorporating reasoning into the RL loop leads to a more substantial improvement, pushing the overall score to 0.72. Notably, categories that require more complex semantic understanding—such as "Colors" (from 0.83 to 0.87) and "Attribute binding" (from 0.46 to 0.53)—benefit the most. These results demonstrate that step-by-step reasoning helps the model better decompose and interpret textual prompts, thereby enabling more accurate and faithful image generation.
\begin{table}[h]
  \caption{Ablation study of reasoning on the GenEval benchmark.}
  \label{tab:abl-reason}
  \centering
  \small
  \setlength{\tabcolsep}{4pt} 
  \begin{tabular}{l|cccccc|c}
    \toprule
    Method  & Single & Two & Counting & Colors & Position & Attribute & Overall~$\uparrow$ \\
           &            object & object &          &        &          & binding   &  \\
    \midrule
    FLUX          & 0.99 & 0.79 & 0.75 & 0.78 & 0.18 & 0.45  & 0.66 \\
    +Qwen2.5 3B      & 0.99 & 0.84 & 0.63 & 0.81 & 0.35 & 0.48 & 0.68 \\
    \midrule
    RL w/o reasoning & 1.00 & 0.81 & 0.68 & 0.83 & 0.33 & 0.46 & 0.68 \\
    RL w/ reasoning & 0.98 & 0.83 & 0.71 & \textbf{0.87} & 0.41 & \textbf{0.53} & 0.72 \\
    
    \bottomrule
  \end{tabular}
\end{table}

\paragraph{Training Dynamics.}
Figure~\ref{fig:line} presents the reward curve during reinforcement learning of the RePrompt. We observe a stable and monotonically increasing trend in the reward, demonstrating that our reward model provides a reliable and effective supervision signal throughout training. The absence of sharp fluctuations or reward collapse suggests that our RL setup maintains stable policy updates. This aligns with the observed downstream improvements, confirming that reward-guided prompt refinement effectively enhances compositional alignment in generated images.
\begin{figure}[t]
  \centering
    \includegraphics[width=0.99\linewidth]{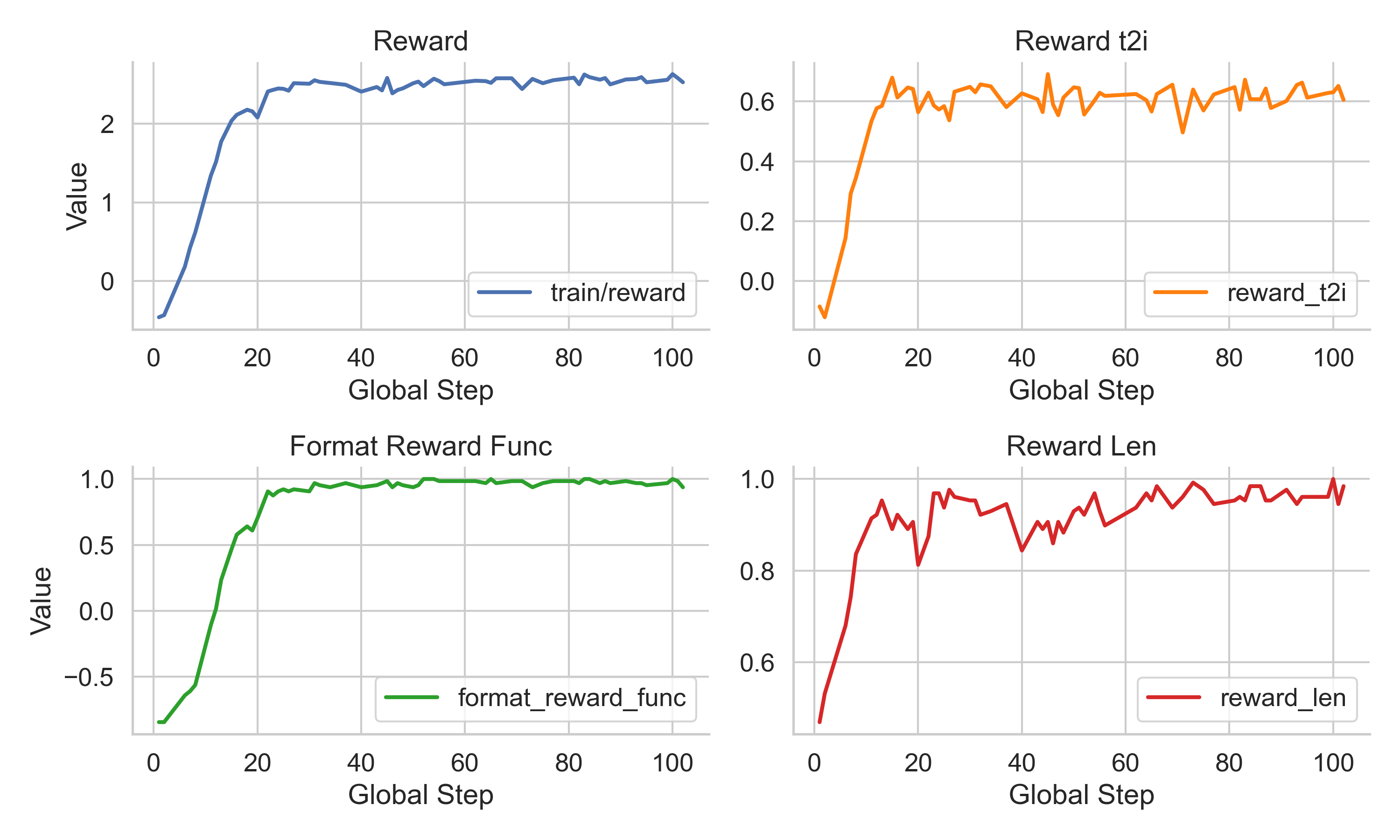}
  \caption{Training curve of our RePrompt during reinforcement learning. The curve shows steady reward improvement, indicating stable training dynamics and effective reward signal.}
  \label{fig:line}
\end{figure}

\paragraph{Ablation Study on Visual Reasoning Rewards.}
To investigate the effectiveness of different visual reasoning reward signals used in our reinforcement learning–based reprompting strategy, we conduct a detailed ablation study on the GenEval benchmark. As shown in Table~\ref{tab:abl-reward}, our method consistently improves over the FLUX baseline and Qwen2.5 3B variant across all subcategories, demonstrating the efficacy of reward-driven learning in aligning generated images with prompt semantics.

Both ImageReward and VLM-Reward show noticeable gains over the +Qwen2.5 3B baseline, particularly in Colors (0.88 vs. 0.81 with ImageReward) and Counting (0.69 vs. 0.63 with VLM-Reward), indicating that each reward captures complementary aspects of visual faithfulness. However, their standalone performances are still limited in challenging tasks like Position and Attribute Binding. Notably, our ensemble reward formulation—which combines both ImageReward and VLM-based feedback—achieves the best overall performance (0.72), with consistent improvements in Counting (+0.08), Position (+0.06), and Attribute Binding (+0.05) compared to the +Qwen2.5 3B baseline. These are precisely the categories that require stronger compositional understanding and spatial reasoning, underscoring the strength of our composite reward design in driving semantically aligned image generation.The ablation confirms that each reward contributes uniquely, and their integration enables more holistic supervision, leading to substantial gains in compositional image-text alignment.

\begin{table}[h]
  \caption{\textbf{Ablation study of rewards on the GenEval benchmark.}}
  \label{tab:abl-reward}
  \centering
  \small
  \setlength{\tabcolsep}{4pt} 
  \begin{tabular}{l|cc|cccccc|c}
    \toprule
    Method  & $\alpha$& $\gamma$& Single & Two & Counting & Colors & Position & Attribute & Overall~$\uparrow$ \\
           & &&           object & object &          &        &          & binding   &  \\
    \midrule
    FLUX       & - & - & 0.99 & 0.79 & 0.75 & 0.78 & 0.18 & 0.45  & 0.66 \\
    +Qwen2.5 3B    & - & - & 0.99 & 0.84 & 0.63 & 0.81 & 0.35 & 0.48 & 0.68 \\
    \midrule
    R1-Prompter & 1 & 0& 0.99 & 0.84 & 0.63 & 0.88 & 0.38 & 0.50 & 0.70 \\
     & 0 & 1 & 0.99 & 0.82 & 0.69 & 0.85 & 0.32 & 0.47 & 0.69 \\
    
     & 0.5 & 0.5 & 0.98 & 0.83 & \textbf{0.71} & 0.87 & \textbf{0.41} & \textbf{0.53} & \textbf{0.72} \\
    
    \bottomrule
  \end{tabular}
\end{table}

\paragraph{Training Dynamics with Different Visual Reasoning Rewards.}
We analyze the impact of different visual reasoning rewards on training dynamics by plotting training curves under various reward configurations. As shown in Figure~\ref{fig:line-abl}, training with a single reward model leads to unstable reward fluctuations and degraded sample quality. This instability arises from the limited supervision capacity of a single reward model, which may overfit to specific patterns or neglect important compositional aspects. In contrast, our proposed multi-reward formulation—balancing diverse reasoning signals, enables smoother optimization and more consistent improvements across iterations. These results emphasize the importance of combining complementary reward models to achieve both stability and generalization.

\begin{figure}[t]
  \centering
    \includegraphics[width=0.99\linewidth]{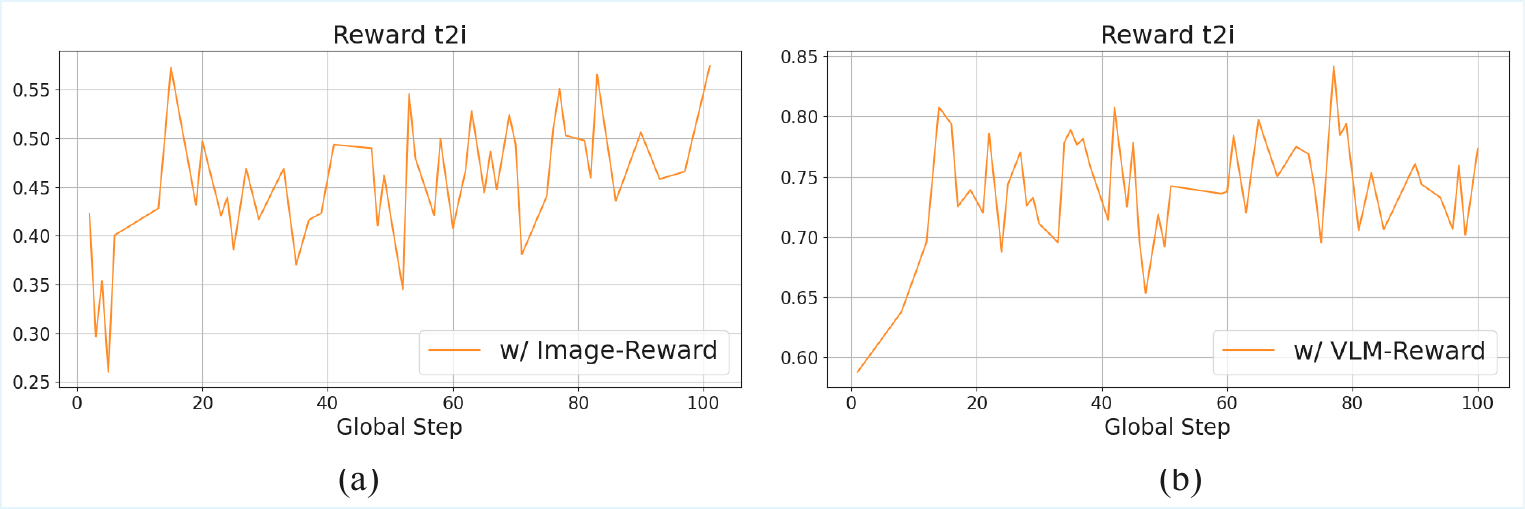}
  \caption{Training curve with different visual reasoning rewards. Using a single reward model leads to instability and suboptimal performance during training. In contrast, our balanced reward design—combining multiple specialized reward signals—yields more stable convergence and higher final reward values.}
  \label{fig:line-abl}
\end{figure}

\begin{figure}[h]
  \centering
    \includegraphics[width=0.96\linewidth]{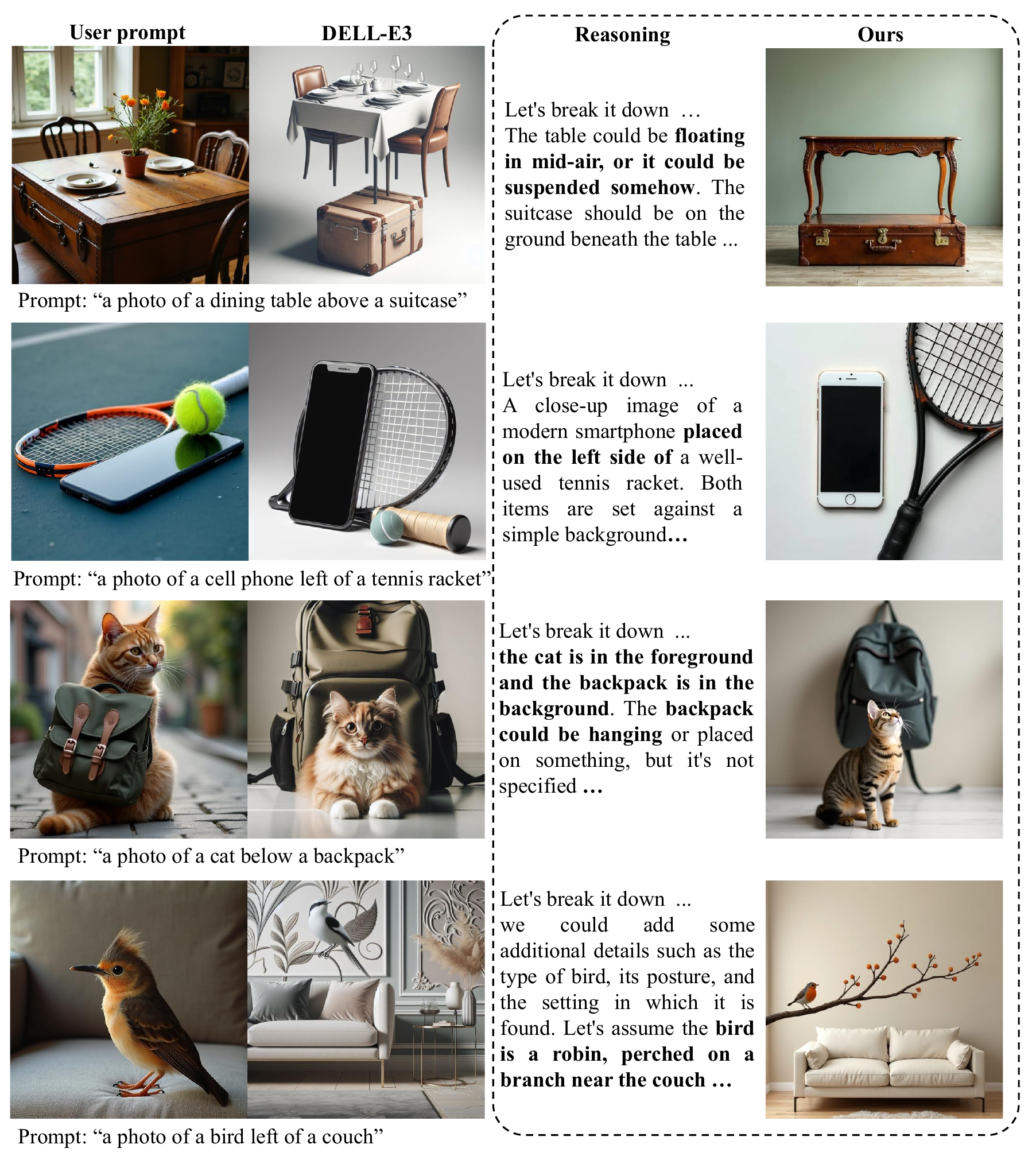}
  \caption{Qualitative results on compositional prompts (Position). We compare DELL-E3, the intermediate reasoning process, and our final RePrompt outputs. While DELL-E3 often struggles with spatial relations (e.g., object positions), our reasoning-guided approach enables more accurate and faithful generations that better match the prompt's compositional constraints.}
  \label{fig:vis1}
\end{figure}

\section{More Cases}
\label{app:d}
\paragraph{More Qualitative Comparison.}
Figure~\ref{fig:vis1},~\ref{fig:vis2},~\ref{fig:vis3} and ~\ref{fig:vis4} present qualitative comparisons on different types of compositional prompts, including Position, Two-object, Color, and Attribute Binding. Compared to baseline methods such as DELL-E3, our approach produces images that more faithfully adhere to the spatial, numerical, and attribute-based constraints specified in the prompts. For instance, in Figure 7, our method accurately grounds relative positions (e.g., "above", "left of") by leveraging explicit reasoning, while DELL-E3 often fails to reflect such relations or misplaces objects entirely. Figure 8 highlights our advantage in handling prompts involving multiple objects, where baseline models tend to merge objects or hallucinate irrelevant content.  Similarly, in Figure 9, our method is better at preserving specified colors for each object, whereas the baseline sometimes misbinds colors or applies them inconsistently. Finally, as shown in Figure 10, our approach improves attribute binding, ensuring that each attribute is applied to the correct object without confusion. These results demonstrate the effectiveness of integrating reasoning to improve alignment between visual outputs and complex prompt semantics.

\begin{figure}[h]
  \centering
    \includegraphics[width=0.99\linewidth]{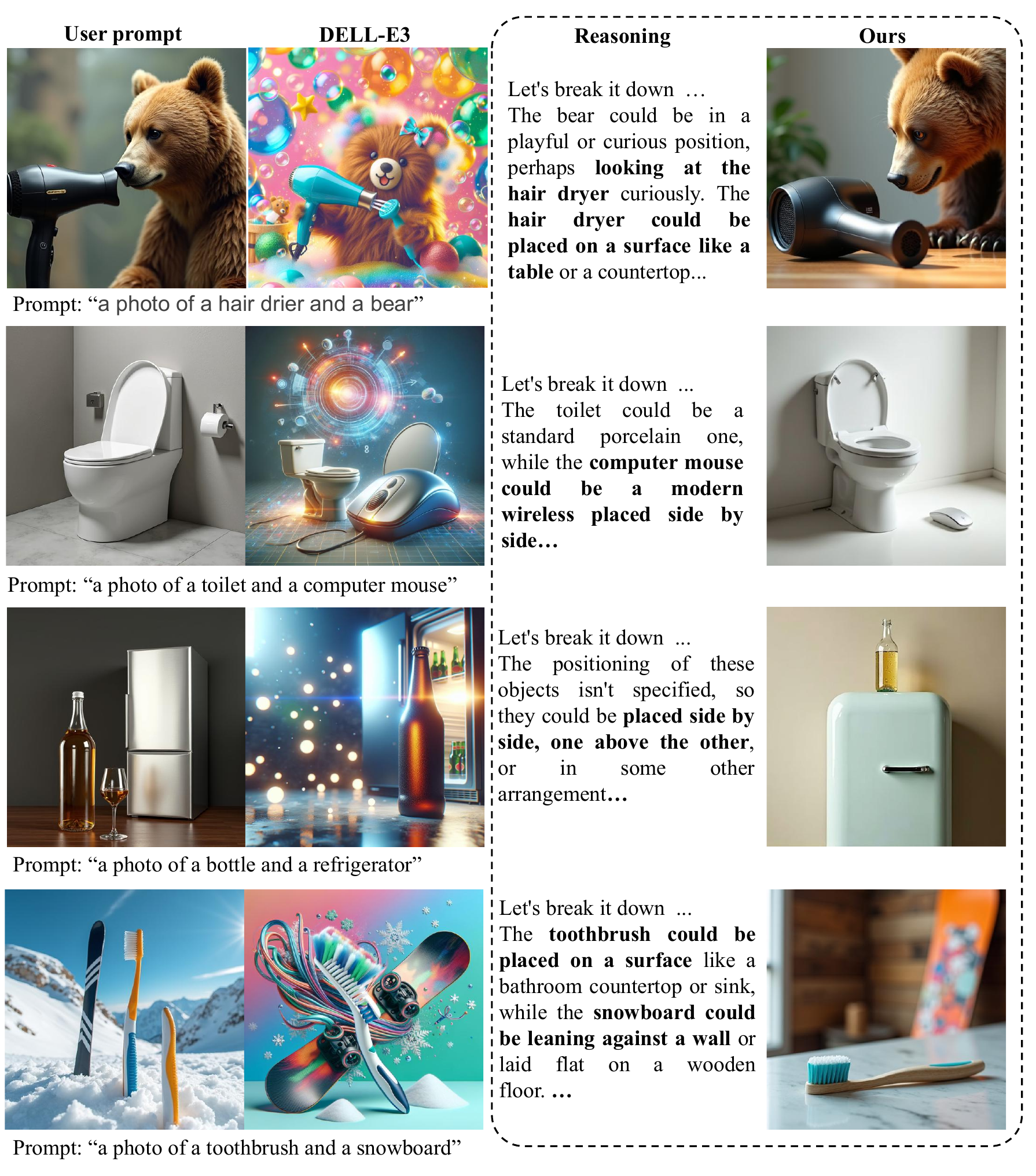}
  \caption{Qualitative results on compositional prompts (Two-object). We show comparisons among DELL-E3, our reasoning process, and RePrompt outputs on prompts involving two objects. DELL-E3 often fails to generate both entities accurately, whereas our method uses explicit reasoning to guide the model in generating semantically correct and compositionally faithful images.}
  \label{fig:vis2}
\end{figure}

\begin{figure}[h]
  \centering
    \includegraphics[width=0.99\linewidth]{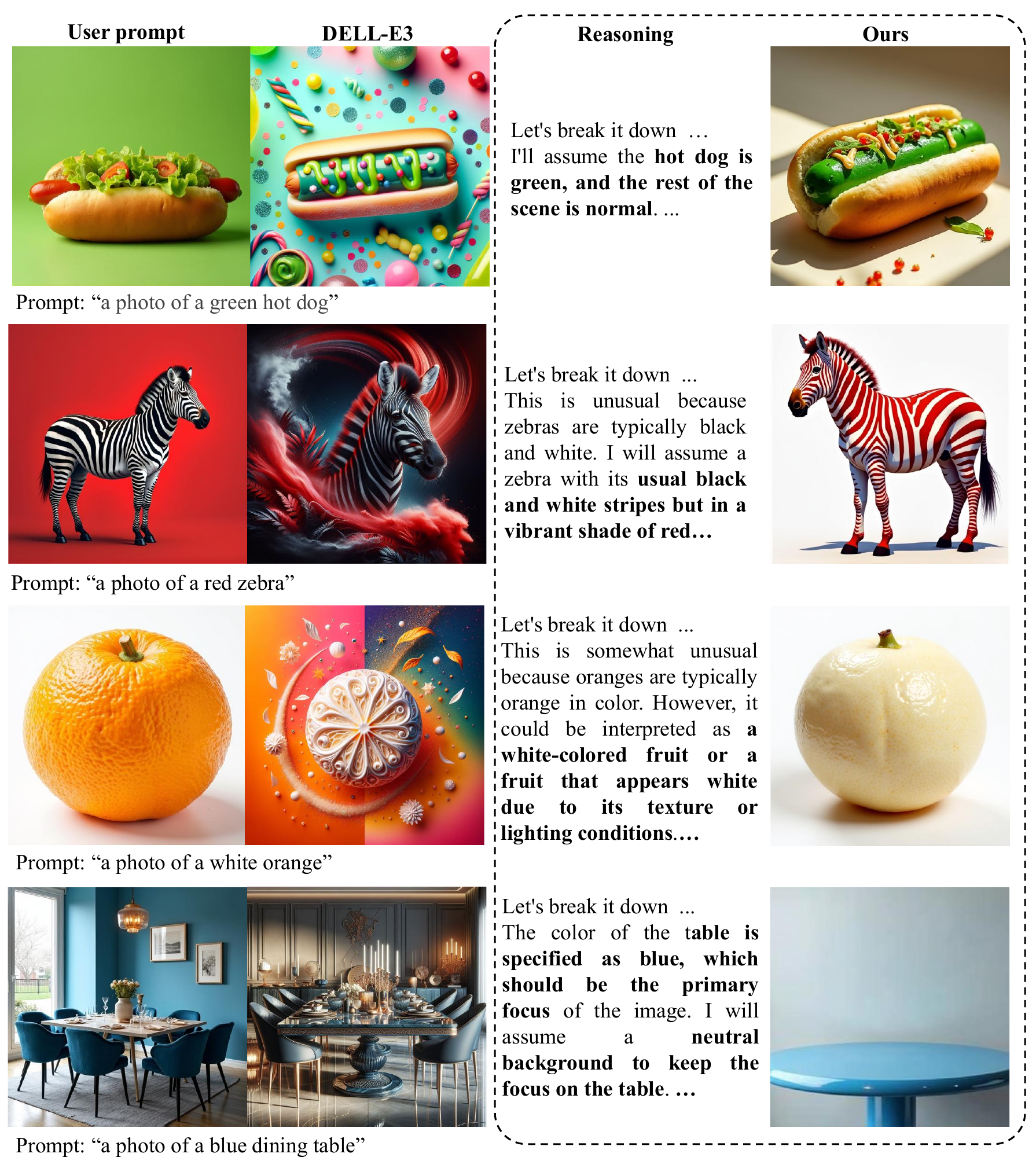}
  \caption{Qualitative results on compositional prompts (Color). We present qualitative comparisons on prompts involving specific color attributes. While DELL-E3 tends to ignore or misinterpret color constraints, our approach leverages explicit reasoning to ensure accurate color grounding for each object, resulting in more faithful visual compositions.}
  \label{fig:vis3}
\end{figure}

\begin{figure}[h]
  \centering
    \includegraphics[width=0.99\linewidth]{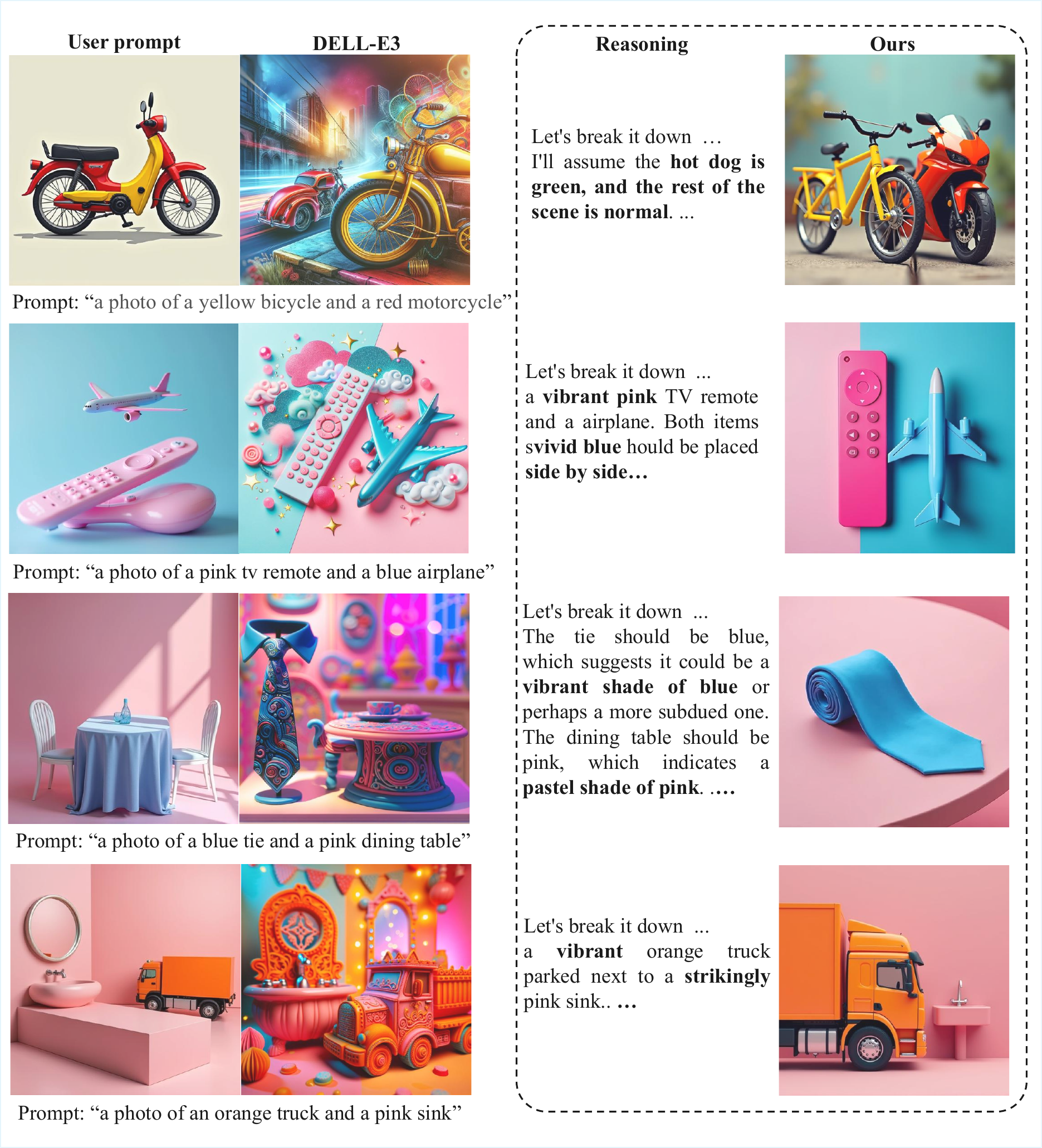}
  \caption{Qualitative results on compositional prompts (Attribute binding). We show examples where prompts specify multiple attributes that must be correctly associated with the corresponding objects. Our method utilizes step-by-step reasoning to disambiguate attribute-object bindings, avoiding attribute swaps or omissions that are common in baseline outputs.}
  \label{fig:vis4}
\end{figure}

\end{document}